\documentclass[runningheads,a4paper]{llncs}
\usepackage{makeidx}  
\usepackage{times}
\usepackage{hyperref}
\usepackage{amssymb,amsfonts,amsmath}
\usepackage[pdftex]{graphicx}
\usepackage{fancyhdr}
\usepackage[]{algorithm2e}
\usepackage{color}

\newcommand{\at}{{a_{t}}}

\newcommand{\st}{{s_{t}}}
\newcommand{\stplus}{{s_{t+1}}}
\newcommand{\lat}{{\theta}}
\newcommand{\latt}{{\theta_t}}

\newcommand{\pol}{{\pi (a|s)}}
\newcommand{\polstar}{{\pi^* (a|s)}}
\newcommand{\polt}{{\pi (\at|\st)}}

\newcommand{\pri}{{\rho (a|s)}}
\newcommand{\prit}{{\rho (\at|\st)}}

\newcommand{\bel}{{\psi(\lat|a,s)}}
\newcommand{\belstar}{{\psi^*(\lat|a,s)}}
\newcommand{\belt}{{\psi(\latt|\at,\st)}}

\newcommand{\bayespost}{{\mu(\lat|a,s)}}
\newcommand{\bayespostt}{{\mu(\latt|\at,\st)}}

\newcommand{\trant}{T(\stplus |\at,\st)}
\newcommand{\tran}{T(s'|a,s)}
\newcommand{\trantheta}{{T_{\theta}(s'|a,s)}}
\newcommand{\thetabold}{{\boldsymbol{\theta}}}

\newcommand{\rew}{R_{s,a}^{s'}}
\newcommand{\rewt}{R_{\st,\at}^{\stplus}}

\newcommand{\disct}{{\gamma^t}}

\newcommand{\argmax}{\operatornamewithlimits{argmax}}
\newcommand{\ext}{\operatornamewithlimits{ext}}
\newcommand{\expec}{{\mathbb{E}}}
\newcommand{\expectedUtilitystar}{{\expec_{\trantheta} \left[ \rew + \gamma F^* (s') \right]}}
\newcommand{\expectedUtilityAlt}{{\expec_{\trantheta} \left[ \rew + \gamma F (s') \right]}}
\newcommand{\expectedUtilityOld}{{\expec_{\trantheta} \left[ \rew + \gamma F_{\textrm{old}} (s') \right]}}
\newcommand{\dkl}{{D_{\text{KL}}}}

\begin{document}
%
%
\pagestyle{headings}  
%

\mainmatter              

\title{Planning with Information-Processing Constraints and Model Uncertainty in Markov Decision Processes}
\titlerunning{Information-Processing Constraints and Model Uncertainty in MDPs}  
%
\author{Jordi Grau-Moya\inst{1,2,3} \and Felix Leibfried\inst{1,2,3} \and
Tim Genewein\inst{1,2,3} \and Daniel~A.~Braun\inst{1,2}}
\authorrunning{Grau-Moya et al.} 
%
\tocauthor{Ivar Ekeland, Roger Temam, Jeffrey Dean, David Grove,
Craig Chambers, Kim B. Bruce, and Elisa Bertino}
\institute{Max Planck Institute for Intelligent Systems, T\"ubingen, Germany,\\
\email{jordi.grau@tuebingen.mpg.de},\\ 
\and
Max Planck Institute for Biological Cybernetics, T\"ubingen, Germany,\\
\and
Graduate Training Centre for Neuroscience, T\"ubingen, Germany
}

\maketitle              

\begin{abstract}
Information-theoretic principles for learning and acting have been proposed to solve particular classes of Markov Decision Problems.
Mathematically, such approaches are governed by a variational free energy principle and allow solving MDP planning problems with information-processing constraints expressed in terms of a Kullback-Leibler divergence with respect to a reference distribution. Here we consider a generalization of such MDP planners by taking model uncertainty into account. As model uncertainty can also be formalized as an information-processing constraint, we can derive a unified solution from a single generalized variational principle. We provide a generalized value iteration scheme together with a convergence proof. As limit cases, this generalized scheme includes standard value iteration with a known model, Bayesian MDP planning, and robust planning. We demonstrate the benefits of this approach in a grid world simulation. 
\end{abstract}

\section{Introduction}

The problem of planning in Markov Decision Processes was famously addressed by Bellman who developed the eponymous principle in 1957 \cite{Bellman:1957}. Since then numerous variants of this principle have flourished in the literature. Here we are particularly interested in a generalization of the Bellman principle that takes information-theoretic constraints into account. In the recent past there has been a special interest in the Kullback-Leibler divergence as a constraint to limit deviations of the action policy from a prior. This can be interesting in a number of ways. Todorov \cite{todorov2006linearly,todorov2009efficient}, for example, has transformed the general MDP problem into a restricted problem class without explicit action variables, where control directly changes the dynamics of the environment and control costs are measured by the Kullback-Leibler divergence between controlled and uncontrolled dynamics. This simplification allows mapping the Bellman recursion to a linear algebra problem. This approach can also be be generalized to continuous state spaces leading to path integral control  \cite{braun2011path,Broek2010RiskSP}. The same equations can also be interpreted in terms of \emph{bounded rational} decision-making where the decision-maker has limited computational resources that allow only limited deviations from a prior decision strategy (measured by the Kullback-Leiber divergence in bits) \cite{ortega2013thermodynamics}. Such a decision-maker can also be instantiated by a sampling process that has restrictions in the number of samples it can afford \cite{ortega2014generalized}. 
Disregarding the possibility of a sampling-based interpretation, the Kullback-Leibler divergence introduces a control information cost that is interesting in its own right when formalizing the perception action cycle \cite{tishby2011information}.

While the above frameworks have led to interesting computational advances, so far they have neglected the possibility of model misspecification in the MDP setting. 
Model misspecification or model uncertainty does not refer to  the uncertainty arising due to the stochastic nature of the  environment (usually called risk-uncertainty in the economic literature), but refers to the uncertainty with respect to the latent variables that specify the MDP. In Bayes-Adaptive MDPs \cite{duff2002optimal}, for example,  the uncertainty over the latent parameters of the MDP is explicitly represented, such that new information can be incorporated with Bayesian inference.  However, Bayes-Adaptive MDPs are not robust with respect to model misspecification and have no performance guarantees when  planning with wrong models \cite{mannor2007bias}. Accordingly, there has been substantial interest in developing robust MDP planners \cite{nilim2005robust,iyengar2005robust,wiesemann2013robust}. One way to take model uncertainty into account is to bias an agent's belief model from a reference Bayesian model towards worst-case scenarios; thus avoiding disastrous outcomes by not visiting states where the transition probabilities are not known. 
Conversely, the belief model can also be biased towards best-case scenarios as a measure to drive exploration---also referred in the literature as \emph{optimism in face of uncertainty} \cite{szita2008many,szita2010model}. 

When comparing the literature on information-theoretic control and model uncertainty, it is interesting to see that some notions of model uncertainty follow exactly the same mathematical principles as the principles of relative entropy control \cite{todorov2009efficient}.
In this paper we therefore formulate a unified and combined optimization problem for MDP planning that takes \emph{both}, model uncertainty and bounded rationality into account. This new optimization problem can be solved by a generalized value iteration algorithm. We provide a theoretical analysis of its convergence properties and simulations in a grid world.

\section{Background and Notation}
In the MDP setting the agent at time $t$ interacts with the environment by taking action $\at \in \mathcal A$ while in state $\st \in \mathcal{S}$. Then the environment updates the state of the agent to $\stplus \in \mathcal{S}$ according to the transition probabilities $\trant$. After each transition the agent receives a reward $\rewt \in \mathcal R$ that is bounded. For our purposes we will consider $\mathcal A$ and $\mathcal S$ to be finite.
The aim of the agent is to choose its policy $\pol$ in order to maximize the total discounted expected reward or value function for any $s \in \mathcal S$
\begin{equation*}
V^* (s) =\max_\pi \lim_{T\rightarrow \infty} \expec \left[\sum_{t=0}^{T-1}\disct   \rewt \right] 
\end{equation*}
with discount factor $0 \le \gamma < 1$. The expectation is over all possible trajectories $\xi = s_0, a_0, s_1 \dots$ of state and action pairs distributed according to $p(\xi) =$ $\prod_{t=0}^{T-1}\polt $ $T(\stplus|\at, \st)$. It can be shown that the optimal value function satisfies the following recursion
\begin{equation}
V^* (s) =  \max_\pi \sum_{a, s'} \pol \tran \left[ \rew + \gamma V^* (s')\right]. \label{eq:valuefunction}
\end{equation}

At this point there are two important implicit assumptions. The first is that the policy $\pi$ can be chosen arbitrarily without any constraints which, for example, might not be true for a bounded rational agent with limited information-processing capabilities. The second  is that the agent needs to know the transition-model $\tran$, but this model is in practice unknown or even misspecified with respect to the environment's true transition-probabilities, specially  at initial stages of learning. In the following, we explain how to incorporate both bounded rationality and model uncertainty into agents.

\subsection{Information-Theoretic Constraints for Acting}
Consider a one-step decision-making problem where the agent is in state $s$ and has to choose a single action $a$ from the set $\mathcal A$ to maximize the reward $\rew$, where $s'$ is the next the state. A perfectly rational agent  selects the optimal action $a^*(s) = \argmax_a \sum_{s'} \tran \rew$. However, a bounded rational agent has only limited resources to find the maximum of the function $\sum_{s'} \tran \rew$. One way to model such an agent is to assume that the agent has a prior choice strategy $\pri$ in state $s$ \emph{before} a deliberation process sets in that refines the choice strategy to a posterior distribution $\pol$ that reflects the strategy \emph{after} deliberation.
Intuitively, because the deliberation resources are limited, the agent can only afford to deviate from the prior strategy by a certain amount of information bits. This can be quantified by the relative entropy $\dkl (\pi||\rho) = \sum_a \pol \log \frac{\pol}{\pri}$ that measures the average information cost of the policy $\pol$  using the source distribution $\pri$. For a bounded rational agent this relative entropy is bounded by some upper limit $K$. Thus, a bounded rational agent has to solve a constrained optimization problem that can be written as
\begin{equation*}
\underset{\pi}{\max} \sum_a \pol \sum_{s'}\tran  \rew ~~~~~~~~ \text{s.t.} ~~\dkl (\pi || \rho) \le K \nonumber
\end{equation*}
This problem can be rewritten as an unconstrained optimization problem 
\begin{align}
F^*(s) &= \max_\pi \sum_a \pol \sum_{s'} \tran \rew - \frac{1}{\alpha}\dkl (\pi||  \rho) \label{eq:free_energy_max}\\ 
&= \frac{1}{\alpha} \log \sum_a \pri e^{\alpha \sum_{s'}\tran \rew} \label{eq:free_energy_nomax}.
\end{align}
where $F^*$ is a free energy that quantifies the value of the policy $\pi$ by trading off the average reward against the information cost. The optimal strategy can be expressed analytically in closed-form as  
\begin{equation*}
	\polstar = \frac{\pri e^{\alpha \sum_{s'}\tran \rew}}{Z_\alpha(s)} 
\end{equation*}
with partition sum $Z_\alpha (s)= \sum_a \pri \exp \left(\alpha  \sum_{s'} \tran \rew \right)$. Therefore, the maximum operator in \eqref{eq:free_energy_max} can be eliminated and the free energy can be rewritten as in \eqref{eq:free_energy_nomax}.
The Lagrange multiplier $\alpha$ quantifies the boundedness of the agent. By setting $\alpha \rightarrow \infty$ we recover a perfectly rational agent with optimal policy $\pi^*(a|s) = \delta (a - a^*(s))$. For $\alpha = 0$ the agent has no computational resources and the agent's optimal policy is to act according to the prior $\pi^*(a|s) = \rho(a|s)$. Intermediate values of $\alpha$ lead to a spectrum of bounded rational agents.

\subsection{Information-Theoretic Constraints for Model Uncertainty}

In the following we assume that the agent has a model of the environment $T_\theta(s'|a,s)$ that depends on some latent variables $\theta \in \Theta$. In the MDP setting, the agent holds a belief $\bayespost$ regarding the environmental dynamics where $\theta$ is a unit vector of transition probabilities into all possible states $s'$.
While interacting with the environment the agent can incorporate new data by forming the Bayesian posterior  $\mu(\theta|a,s, D)$, where $D$ is the observed data.
When the agent has observed an infinite amount of data (and assuming $\theta^*(a,s) \in \Theta$) the  belief will converge to the delta distribution  $\mu(\theta|s,a,D)= \delta (\theta -\theta^*(a,s))$ and the agent will act optimally according to the true transition probabilities, exactly as in ordinary optimal choice strategies with known models.
When acting under a limited amount of data the agent cannot determine the value of an action $a$ with the true transition model according to $\sum_{s'} \tran \rew$, but it can only determine an expected value according to its beliefs $\int_\theta \bayespost \sum_{s'} \trantheta \rew$. 

The Bayesian model $\mu$ can be subject to model misspecification (e.g. by  having a wrong likelihood or a bad prior) and thus the agent  might want to allow deviations from its model towards best-case (optimistic agent) or worst-case (pessimistic agent) scenarios up to a certain extent, in order to act more robustly or to enhance its performance in a friendly environment \cite{hansen2008robustness}. Such deviations can be measured by the relative entropy $\dkl (\psi|\mu)$ between the Bayesian posterior $\mu$ and a new biased model $\psi$. Effectively, this allows for mathematically formalizing model uncertainty, by not only considering the specified model but all models within a neighborhood of the specified model that deviate no more than a restricted number of bits. Then, the effective expected value of an action $a$ while having limited trust in the Bayesian posterior $\mu$ can be determined for the case of optimistic deviations as
\begin{equation}
F^* (a,s) = \max_\psi \int_\theta \bel \sum_{s'}\trantheta \rew - \frac{1}{\beta}\dkl(\psi||\mu)\label{eq:FEbeliefs_up}
\end{equation}
for $\beta >0$, and for the case of pessimistic deviations as
\begin{equation}
F^* (a,s) =\min_\psi \int_\theta \bel \sum_{s'}\trantheta \rew  - \frac{1}{\beta}\dkl(\psi||\mu) \label{eq:FEbeliefs_down}
\end{equation}
for $\beta<0$. Conveniently, both equations can be expressed as a single equation 
\begin{equation*}
F^*(a,s) = \frac{1}{\beta} \log Z_\beta (a,s) 
\end{equation*}
 with $\beta \in \mathbb{R}$ and $Z_\beta(s,a) = \int_\theta \bayespost \exp \left(\beta \sum_{s'} \trantheta \rew  \right)$ when inserting the optimal biased belief 
\begin{equation*}
\psi^*(\theta|a,s) = \frac{1}{Z_\beta(a,s)} \bayespost \exp \left( \beta \sum_{s'} \trantheta \rew \right) 
\end{equation*}
into either equation \eqref{eq:FEbeliefs_up} or \eqref{eq:FEbeliefs_down}. 
By adopting this formulation we can model any degree of trust in the belief $\mu$ allowing deviation towards worst-case or best-case with $-\infty \le \beta \le \infty$. For the case of $\beta \rightarrow -\infty $ we recover an infinitely pessimistic agent that considers only worst-case scenarios, for $\beta \rightarrow \infty$  an agent that is infinitely optimistic and for $\beta \rightarrow 0$ the Bayesian agent that fully trusts its model.

\section{Model Uncertainty and Bounded Rationality in MDPs}

In this section, we consider a bounded rational agent with model uncertainty in the infinite horizon setting of an MDP. In this  case the agent must take into account all future rewards and  information costs, thereby optimizing the following free energy objective
\begin{multline}\label{eq:FEobjective}
F^* (s) = \max_{\pi}\ext_{\psi} \lim_{T\rightarrow \infty} \expec \sum_{t=0}^{T-1} \disct\Bigg(   \rewt -\frac{1}{\beta} \log \frac{\psi(\latt| \at, \st)}{\bayespostt}
 -\frac{1}{\alpha} \log \frac{\polt}{\prit} \Bigg)
\end{multline}
where the extremum operator $\ext$ can be either $\max$ for $\beta >0$ or $\min$ for $\beta<0$,   $0 < \gamma<1$ is the discount factor and the expectation $\expec$ is over all trajectories $\xi = s_0, a_0,\theta_0,s_1,a_1,\dots a_{T-1}, \theta_{T-1},s_{T}$ with distribution $p(\xi) =$ $\prod_{t=0}^{T-1}\polt $ $\belt $ $T_{\theta_t}(\stplus|\at, \st)$. 
Importantly, this free energy objective satisfies a recursive relation and thereby generalizes Bellman's optimality principle to the case of model uncertainty and bounded rationality. In particular, equation \eqref{eq:FEobjective} fulfills the recursion
\begin{multline}
F^*(s) = \max_{\pi} \ext_{\psi} \expec_{\pol} \Bigg[ -  \frac{1}{\alpha} \log \frac{\pol}{\pri} + \\
 \expec_{\bel} \bigg[  - \frac{1}{\beta} \log \frac{\bel}{\bayespost} + \\
  \expectedUtilitystar \bigg]  \Bigg].\label{eq:recursion1}
\end{multline}

Applying variational calculus and following the same rationale as in the previous sections \cite{ortega2013thermodynamics}, the extremum operators can be eliminated and equation $\eqref{eq:recursion1}$ can be re-expressed as 
\begin{equation}
F^*(s) = \frac{1}{\alpha}\log  \expec_{\pri}  \left[ \expec_{\bayespost} \left[ \exp\left(\beta \expectedUtilitystar \right) \right]^{\frac{\alpha}{\beta}} \right]\label{eq:nicerecursion}
\end{equation}
because
\begin{align}
F^*(s) &= \max_\pi \expec_\pol \left[ \frac{1}{\beta} \log Z_\beta(a,s) - \frac{1}{\alpha} \log \frac{ \pol}{\pri} \right] \label{eq:felix} \\
&= \frac{1}{\alpha}\log  \expec_{\pri}  \left[ \exp\left( \frac{\alpha}{\beta} \log Z_\beta (a,s) \right) \right],\label{eq:FE1var}
\end{align}
where
\begin{align}
Z_\beta(a,s) & =  \ext_\psi \expec_{\bel} \bigg[ \expectedUtilitystar - \frac{1}{\beta} \log \frac{\bel}{\bayespost} \bigg] \label{eq:Z2var} \\
& = \expec_{\bayespost} \exp\left(\beta \expectedUtilitystar \right) \nonumber
\end{align}
with the optimizing arguments
\begin{equation*}
\belstar = \frac{1}{Z_\beta(a,s)}\bayespost \exp \left(\beta \expectedUtilityAlt \right)
\end{equation*}
\begin{equation}
\polstar = \frac{1}{Z_\alpha (s)}\pri \exp \left(\frac{\alpha }{\beta} \log Z_\beta (a,s)\right)\label{eq:policy_equilibrium}
\end{equation}
 and partition sum 
\begin{equation*}
Z_\alpha(s) = \expec_\pri \left[\exp  \left(  \frac{\alpha}{\beta} \log Z_\beta (a,s) \right) \right]. 
\end{equation*}

With this free energy we can model a range of different agents for different $\alpha$ and $\beta$. For example, by setting $\alpha \rightarrow \infty$ and $\beta \rightarrow 0$ we can recover a Bayesian MDP planner and by setting $\alpha \rightarrow \infty$ and $\beta \rightarrow -\infty$ we recover a robust planner. Additionally, for $\alpha \rightarrow \infty$ and when $\bayespost=\delta(\theta - \theta^*(a,s))$ we recover an agent with standard value function with known state transition model from equation \eqref{eq:valuefunction}.

\subsection{Free Energy Iteration Algorithm}

Solving the self-consistency equation \eqref{eq:nicerecursion} can be achieved by a generalized version of value iteration. Accordingly, the optimal solution can be obtained by initializing the free energy at some arbitrary value $F$ and applying a value iteration scheme $B^{i+1}F = BB^i F$ where we define the operator
 \begin{multline} \label{eq:recursion2}
 BF (s) = \max_{\pi} \ext_{\psi} \expec_{\pol} \Bigg[ -  \frac{1}{\alpha} \log \frac{\pol}{\pri} +\\
  \expec_{\bel} \bigg[  - \frac{1}{\beta} \log \frac{\bel}{\bayespost} + \\
   \expectedUtilityAlt \bigg]  \Bigg]
 \end{multline}
with $B^1F = BF$, which can be simplified to
\begin{equation*}
 BF(s) = \frac{1}{\alpha}\log  \expec_{\pri}  \left[ \expec_{\bayespost} \left[ \exp\left(\beta \expectedUtilityAlt \right) \right]^{\frac{\alpha}{\beta}} \right]
 \end{equation*}
In Algorithm \eqref{alg:FEiteration} we show the pseudo-code of this generalized value iteration scheme. Given state-dependent prior policies $\pri$ and the Bayesian posterior beliefs $\bayespost$ and the values of $\alpha$ and $\beta$, the algorithm outputs the equilibrium distributions for the action probabilities $\pol$, the biased beliefs $\bel$ and estimates of the free energy value function $F^*(s)$. The iteration is run until a convergence criterion is met. The convergence proof is shown in the next section.

\RestyleAlgo{boxruled}
\begin{algorithm}[h!]
 \SetKwInput{KwInitialize}{Initialize}
 \KwIn{ $\pri, \bayespost, \alpha, \beta$}
 \KwInitialize{  $F \leftarrow 0$, $F_{\textrm{old}} \leftarrow 0$}
\While{not converged}{
  \ForAll{$s \in \mathcal S$}{
  $F(s) \leftarrow \frac{1}{\alpha}\log  \expec_{\pri}  \left[ \expec_{\bayespost} \left[ \exp\left(\beta \expectedUtilityOld \right) \right]^{\frac{\alpha}{\beta}} \right]$ 
  } $F_{\textrm{old}} \leftarrow F $
}
$\pol \leftarrow \frac{1}{Z_\alpha (s)}\pri \exp \left(\frac{\alpha }{\beta} \log Z_\beta (a,s)\right)$\\
$\bel \leftarrow \frac{1}{Z_\beta(a,s)}\bayespost \exp\left({\beta \expectedUtilityAlt}\right)
$\\
 \KwRet{$\pol$, $\bel$, $F(s)$}
\BlankLine
\label{alg:FEiteration}
 \caption{Iterative algorithm solving the self-consistency equation \eqref{eq:nicerecursion}}
\end{algorithm}

\section{Convergence}\label{sec:proofs}

Here, we show that the value iteration scheme described through Algorithm~\ref{alg:FEiteration} converges to a unique fixed point satisfying Equation~\eqref{eq:nicerecursion}. To this end, we first prove the existence of a unique fixed point (Theorem~\ref{theo:fixedpoint}) following \cite{rubin2012trading,bertsekas1996neuro}, and subsequently prove the convergence of the value iteration scheme presupposing that a unique fixed point exists (Theorem~\ref{theo:conv}) following \cite{strehl2009reinforcement}.

\begin{theorem} \label{theo:fixedpoint}
Assuming a bounded reward function $\rew$, the optimal free-energy vector $F^*(s)$ is a unique fixed point of Bellman's equation $F^*=BF^*$, where the mapping $B:\mathbb{R}^{|\mathcal{S}|} \rightarrow \mathbb{R}^{|\mathcal{S}|}$ is defined as in equation \eqref{eq:recursion2}
\end{theorem}

\begin{proof} 
Theorem~\ref{theo:fixedpoint} is proven through Proposition~\ref{theo:fixedpointprop} and~\ref{theo:fixedpointprop2} in the following. 
\end{proof}
\begin{proposition} \label{theo:fixedpointprop}
The mapping $T_{\pi,\psi}: \mathbb{R}^{|\mathcal{S}|} \rightarrow \mathbb{R}^{|\mathcal{S}|}$
\begin{multline}\label{eq:recursionpol}
T_{\pi,\psi}F (s) = \expec_{\pol} \Bigg[ -  \frac{1}{\alpha} \log \frac{\pol}{\pri} +\\
 \expec_{\bel} \bigg[  - \frac{1}{\beta} \log \frac{\bel}{\bayespost} + \\
  \expectedUtilityAlt \bigg]  \Bigg].
\end{multline}
converges to a unique solution for every policy-belief-pair $(\pi,\psi)$ independent of the initial free-energy vector $F(s)$.
\end{proposition}
\begin{proof}
 By introducing the matrix $P_{\pi,\psi}(s,s')$ and the vector $g_{\pi,\psi}(s)$ as
\begin{equation*}
P_{\pi,\psi}(s,s') := \mathbb{E}_{\pol} \bigg[ \mathbb{E}_{\psi(\theta|a,s)} \left[ T_\theta(s'|a,s)\right] \bigg] ,
\end{equation*}
\begin{equation*}
g_{\pi,\psi}(s) := \mathbb{E}_{\pol} \Bigg[ \mathbb{E}_{\psi(\theta|a,s)} \bigg[ \mathbb{E}_{T_\theta(s'|a,s)} \left[ \rew \right] - \frac{1}{\beta} \log\frac{\psi(\theta|a,s)}{\mu(\theta|a,s)} \bigg] - \frac{1}{\alpha} \log\frac{\pol}{\pri} \Bigg],
\end{equation*}
Equation~\eqref{eq:recursionpol} may be expressed in compact form: $T_{\pi,\psi}F = g_{\pi,\psi} + \gamma P_{\pi,\psi} F$. By applying the mapping $T_{\pi,\psi}$ an infinite number of times on an initial free-energy vector $F$, the free-energy vector $F_{\pi,\psi}$ of the policy-belief-pair $(\pi,\psi)$ is obtained:
\begin{equation*}
F_{\pi,\psi} := \lim_{i\rightarrow \infty} T_{\pi,\psi}^i F = \lim_{i\rightarrow \infty} \sum_{t=0}^{i-1} \gamma^t P_{\pi,\psi}^t g_{\pi,\psi} + \underbrace{\lim_{i\rightarrow \infty} \gamma^i P_{\pi,\psi}^i F}_{\rightarrow 0} ,
\end{equation*}
which does no longer depend on the initial $F$. It is straightforward to show that the quantity $F_{\pi,\psi}$ is a fixed point of the operator $T_{\pi,\psi}$:
\begin{align*}
T_{\pi,\psi}F_{\pi,\psi}  & =  g_{\pi,\psi} + \gamma P_{\pi,\psi} \lim_{i\rightarrow \infty} \sum_{t=0}^{i-1} \gamma^t P_{\pi,\psi}^t g_{\pi,\psi} \nonumber\\
& =  \gamma^0 P_{\pi,\psi}^0 g_{\pi,\psi} + \lim_{i\rightarrow \infty} \sum_{t=1}^{i} \gamma^t P_{\pi,\psi}^t g_{\pi,\psi} \nonumber\\
 & =  \lim_{i\rightarrow \infty} \sum_{t=0}^{i-1} \gamma^t P_{\pi,\psi}^t g_{\pi,\psi} + \underbrace{\lim_{i\rightarrow \infty} \gamma^i P_{\pi,\psi}^i g_{\pi,\psi}}_{\rightarrow 0} = F_{\pi,\psi}.
\end{align*}
Furthermore, $F_{\pi,\psi}$ is unique. Assume for this purpose an arbitrary fixed point $F'$ such that $T_{\pi,\psi}F' = F'$, then $F' = \lim_{i\rightarrow\infty}T_{\pi,\psi}^iF'=F_{\pi,\psi}$. 
\end{proof}
\begin{proposition} \label{theo:fixedpointprop2}
The optimal free-energy vector $F^*=\max_{\pi}\ext_{\psi}F_{\pi,\psi}$ is a unique fixed point of Bellman's equation $F^*=BF^*$.
\end{proposition}
\begin{proof} The proof consists of two parts where we assume $\ext = \max$ in the first part and $\ext = \min$ in the second part respectively. Let $\ext = \max$ and $F^*=F_{\pi^*,\psi^*}$, where $(\pi^*,\psi^*)$ denotes the optimal policy-belief-pair. Then
\begin{equation*}
F^* = T_{\pi^*,\psi^*} F^* \leq \underbrace{\max_\pi \max_\psi T_{\pi,\psi} F^*}_{=BF^*} =: T_{\pi',\psi'} F^* \stackrel{\text{Induction}}{\leq} F_{\pi',\psi'},
\end{equation*}
where the last inequality can be straightforwardly proven by induction and exploiting the fact that $P_{\pi,\psi}(s,s') \in [0;1]$. But by definition $F^* = \max_{\pi}\max_{\psi}F_{\pi,\psi} \geq F_{\pi',\psi'}$, hence $F^* = F_{\pi',\psi'}$ and therefore $F^*=BF^*$. Furthermore, $F^*$ is unique. Assume for this purpose an arbitrary fixed point $F'=F_{\pi',\psi'}$ such that $F'=BF'$ with the corresponding policy-belief-pair $(\pi',\psi')$. Then 
\begin{equation*}
F^* = T_{\pi^*,\psi^*} F^* \geq T_{\pi',\psi'} F^* \stackrel{\text{Induction}}{\geq} F_{\pi',\psi'} = F',
\end{equation*}
and similarly $F' \geq F^*$, hence $F' = F^*$.

Let $\ext = \min$ and $F^*=F_{\pi^*,\psi^*}$. By taking a closer look at Equation~\eqref{eq:recursion2}, it can be seen that the optimization over $\psi$ does not depend on $\pi$. Then
\begin{equation*}
F^* = T_{\pi^*,\psi^*} F^* \geq \min_{\psi} T_{\pi^*,\psi} F^* =: T_{\pi^*,\psi'} F^* \stackrel{\text{Induction}}{\geq} F_{\pi^*,\psi'}.
\end{equation*}
But by definition $F^*=\min_{\psi}F_{\pi^*,\psi} \leq F_{\pi^*,\psi'}$, hence $F^*=F_{\pi^*,\psi'}$. Therefore it holds that $BF^* = \max_\pi \min_\psi T_{\pi,\psi} F^* = \max_\pi T_{\pi,\psi^*} F^*$ and similar to the first part of the proof we obtain
\begin{equation*}
F^* = T_{\pi^*,\psi^*} F^* \leq \underbrace{\max_\pi T_{\pi,\psi^*} F^*}_{= BF^*} =: T_{\pi',\psi*} F^* \stackrel{\text{Induction}}{\leq} F_{\pi',\psi*}.
\end{equation*}
But by definition $F^* = \max_{\pi}F_{\pi,\psi^*} \geq F_{\pi',\psi*}$, hence $F^*=F_{\pi',\psi*}$ and therefore $F^*=BF^*$. Furthermore, $F_{\pi^*, \psi^*}$ is unique. Assume for this purpose an arbitrary fixed point $F'=F_{\pi',\psi'}$ such that $F'=BF'$. Then
\begin{equation*}
F' = T_{\pi',\psi'} F' \leq T_{\pi',\psi^*} F' \stackrel{\text{Induction}}{\leq} F_{\pi', \psi^*} \stackrel{\text{Induction}}{\leq} T_{\pi',\psi^*} F^* \leq T_{\pi^*,\psi^*} F^* = F^*,
\end{equation*}
and similarly $F^* \leq F'$, hence $F^* = F'$. 
\end{proof}

\begin{theorem}\label{theo:conv}
Let $\epsilon$ be a positive number satisfying $\epsilon<\frac{\eta}{1-\gamma}$ where $\gamma \in (0;1)$ is the discount factor and where $u$ and $l$ are the bounds of the reward function $\rew$ such that $l \leq \rew \leq u$ and $\eta=\max\{|u|,|l|\}$. Suppose that the value iteration scheme from Algorithm~\ref{alg:FEiteration} is run for $i=\lceil \log_\gamma\frac{\epsilon(1-\gamma)}{\eta} \rceil$ iterations with an initial free-energy vector $F(s)=0$ for all $s$. Then, it holds that $\max_s |F^*(s) - B^iF(s)| \leq \epsilon$, where $F^*$ refers to the unique fixed point from Theorem~\ref{theo:fixedpoint}.
\end{theorem}
\begin{proof}
We start the proof by showing that the $L_\infty$-norm of the difference vector between the optimal free-energy $F^*$ and $B^iF$ exponentially decreases with the number of iterations $i$:
\begin{align*}
   \max_s & \left|F^*(s) - B^iF(s)\right| =: \left|F^*(s^*) - B^iF(s^*)\right| \nonumber\\
&  \stackrel{\text{Eq. } \eqref{eq:felix}}{=} \left| \max_\pi \mathbb{E}_{\pi(a|s^*)} \bigg[ \frac{1}{\beta} \log Z_\beta(a,s^*) - \frac{1}{\alpha} \log \frac{\pi(a|s^*)}{\rho(a|s^*)}\bigg] \right. \nonumber\\
& \left. - \max_\pi \mathbb{E}_{\pi(a|s^*)} \bigg[ \frac{1}{\beta} \log Z^i_\beta(a,s^*) - \frac{1}{\alpha} \log \frac{\pi(a|s^*)}{\rho(a|s^*)}\bigg] \right| \nonumber\\
& \leq \max_\pi \left| \mathbb{E}_{\pi(a|s^*)} \bigg[ \frac{1}{\beta} \log Z_\beta(a,s^*) - \frac{1}{\beta} \log Z^i_\beta(a,s^*) \bigg] \right| \nonumber\\
& \leq\max_a \left|  \frac{1}{\beta} \log Z_\beta(a,s^*) - \frac{1}{\beta} \log Z^i_\beta(a,s^*) \right| \nonumber\\
& =: \left|  \frac{1}{\beta} \log Z_\beta(a^*,s^*) - \frac{1}{\beta} \log Z^i_\beta(a^*,s^*) \right| \nonumber \\
& \stackrel{\text{Eq. } \eqref{eq:Z2var}}{=} \left| \ext_\psi \mathbb{E}_{\psi(\theta|a^*,s^*)} \bigg[ \mathbb{E}_{T_\theta(s'|a^*,s^*)} \big[ \rew + \gamma F^*(s') \big] -\frac{1}{\beta} \log \frac{\psi(\theta|a^*,s^*)}{\mu(\theta|a^*,s^*)}  \bigg] \right. \nonumber \\
& - \left. \ext_\psi \mathbb{E}_{\psi(\theta|a^*,s^*)} \bigg[ \mathbb{E}_{T_\theta(s'|a^*,s^*)} \big[ \rew + \gamma B^{i-1}F(s') \big] -\frac{1}{\beta} \log \frac{\psi(\theta|a^*,s^*)}{\mu(\theta|a^*,s^*)}  \bigg] \right| \nonumber \\
& \leq \max_\psi \left| \mathbb{E}_{\psi(\theta|a^*,s^*)} \bigg[ \mathbb{E}_{T_\theta(s'|a^*,s^*)} \big[ \gamma F^*(s') -  \gamma B^{i-1}F(s')\big] \bigg] \right| \nonumber\\
& \leq \gamma \max_s \left| F^*(s) - B^{i-1}F(s) \right| \stackrel{\text{Recur.}}{\leq} \gamma^i \max_s \left| F^*(s) - F(s) \right| \leq \gamma^i \frac{\eta}{1-\gamma} \nonumber,
\end{align*}
where we exploit the fact that $\left| \ext_xf(x) - \ext_xg(x) \right| \leq \max_x \left|f(x) - g(x) \right|$ and that the free-energy is bounded through the reward bounds $l$ and $u$ with  $\eta=\max\{|u|,|l|\}$. For a convergence criterion $\epsilon>0$ such that $\epsilon \geq \gamma^i \frac{\eta}{1-\gamma}$, it then holds that $i \geq \log_\gamma \frac{\epsilon(1-\gamma)}{\eta}$ presupposing that $\epsilon < \frac{\eta}{1-\gamma}$.
\end{proof}

\section{Experiments: Grid World}
This section illustrates the proposed value iteration scheme with an intuitive example where an agent has to navigate through a grid-world. The agent starts at  position $\textbf{S} \in \mathcal{S}$ with the objective to reach the goal state $\textbf{G} \in \mathcal{S}$ and  can choose one out of maximally four possible actions $a \in \lbrace \uparrow, \rightarrow, \downarrow, \leftarrow \rbrace $ in each time-step. Along the way, the agent can encounter regular tiles (actions move the agent deterministically one step in the desired direction), walls that are represented as \emph{gray tiles} (actions that move  the agent towards the wall are not possible), holes that are represented as \emph{black tiles} (moving into the hole causes a negative reward) and \emph{chance tiles} that are illustrated as white tiles with a question mark (the transition probabilities of the chance tiles are unknown to the agent). Reaching the goal $\textbf{G}$ yields a reward $R=+1$ whereas stepping into a hole results in a negative reward $R= -1$. In both cases the agent is subsequently teleported back to the starting position $\textbf{S}$. Transitions to regular tiles have a small negative reward of $R= -0.01$. When stepping onto a chance tile, the agent is pushed stochastically to an adjacent tile giving a reward as mentioned above. 
The true state-transition probabilities of the chance tiles are not known by the agent, but the agent holds the Bayesian belief 
\[
\mu (\thetabold_{s,a} |a,s) = \mathsf{Dirichlet} \big( \Phi_{s,a}^{s_1'}, \dots, \Phi_{s,a}^{s_{N(s)}'}\big) = \prod_{i=1}^{N(s)} (\theta_{s,a}^{s_i'})^{\Phi_{s,a}^{s_i'} - 1}
\]
where transition model is denoted as $T_{\thetabold_{s,a}} (s'|s,a)= \theta_{s,a}^{s'}$
and $\thetabold_{s,a} = \big(\theta_{s,a}^{s_1'} \dots \theta_{s,a}^{s_{N(s)}'} \big)$ and $N(s)$ is the number of possible actions in state $s$. The data  is incorporated into the model as a count vector $\big( \Phi_{s,a}^{s_1'}, \dots, \Phi_{s,a}^{s_{N(s)}'}\big) $ where $\Phi_{s,a}^{s' }$ represents the number of times that the transition $(s,a,s')$ has occurred. The prior $\pri$ for the actions at every state is set to be uniform. An important aspect of the model is that in the case of unlimited observational data, the agent will plan with the correct transition probabilities.

We conducted two experiments with discount factor $\gamma=0.9$ and uniform priors $\pri$ for the action variables. In the first experiment, we explore and illustrate the agent's  planning behavior under different degrees of computational limitations (by varying $\alpha$) and under different model uncertainty attitudes (by varying $\beta$) with fixed uniform beliefs $\bayespost$. 
In the second experiment, the agent is allowed to update its beliefs $\bayespost$ and use the updated model to re-plan its strategy.

\subsection{The Role of the Parameters $\alpha$ and $\beta$ on Planning}

Figure \ref{fig:results1} shows the solution to the variational free energy problem that is obtained by iteration until convergence according to Algorithm \ref{alg:FEiteration}  under different values of $\alpha$ and $\beta$. In particular, the first row shows the free energy function $F^*(s)$ (Eq.~\eqref{eq:nicerecursion}). The second, third and fourth row show heat maps of the position of an agent that follows the optimal policy (Eq.~\eqref{eq:policy_equilibrium}) according to the agent's biased beliefs (plan) and to the actual transition probabilities in a friendly and  unfriendly environment, respectively. In chance tiles, the most likely transitions in these two environments are indicated by arrows where the agent is teleported with a probability of $0.999$ into the tile indicated by the arrow and with a probability of $0.001$ to a random other adjacent tile.

\begin{figure}
\includegraphics[trim=25 0 0 0, scale=0.39]{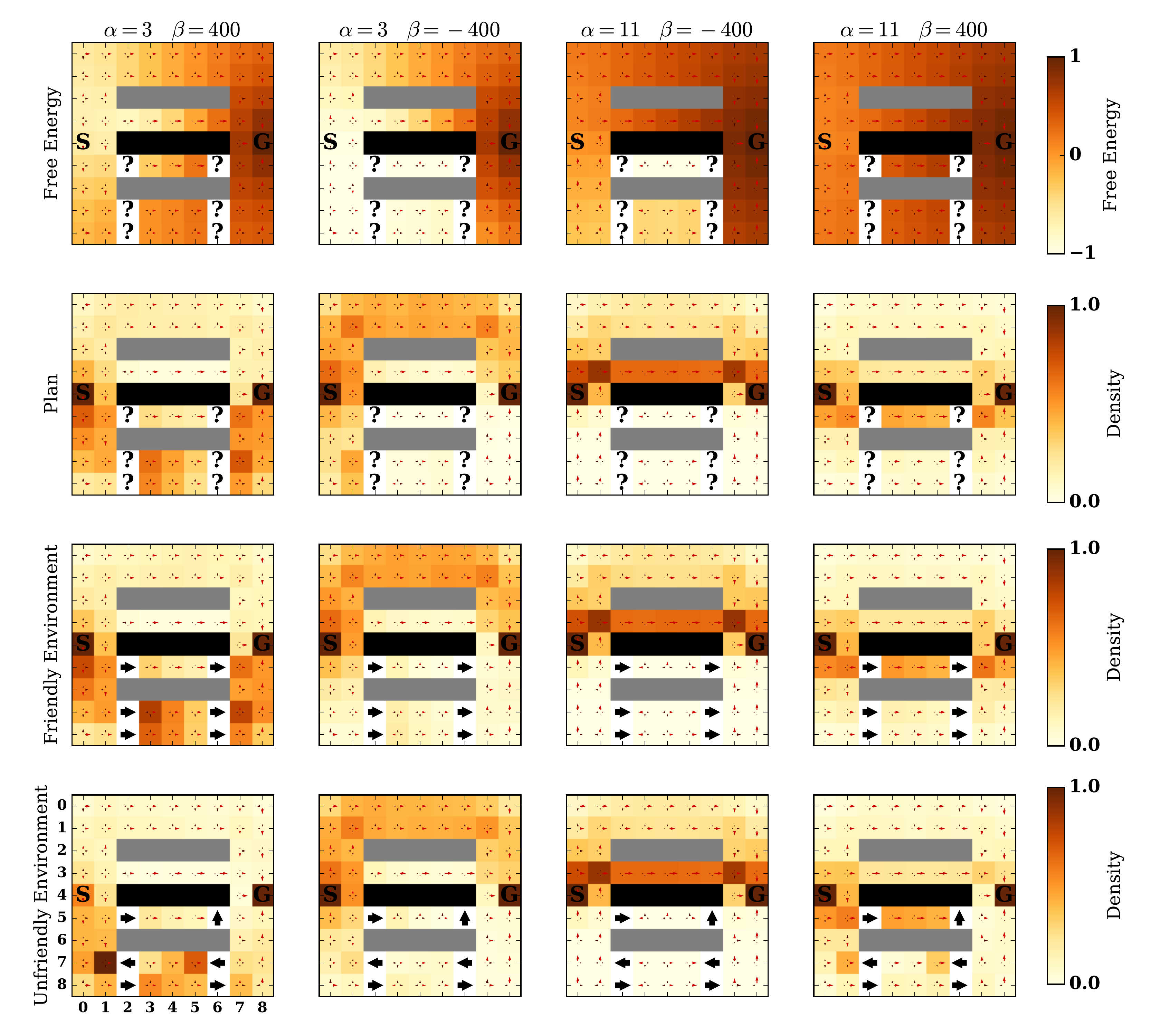}
\caption{The four different rows show free energy values and heat-maps of planned trajectories according to the agent's beliefs over state-transitions in chance tiles, heat-maps of real trajectories in a friendly environment and in an unfriendly environment respectively. The Start-position is indicated by $\textbf{S}$ and the goal state is indicated by $\textbf{G}$. Black tiles represent holes with negative reward, gray tiles represent walls and chance tiles with a question mark have transition probabilities unknown to the agent. The white tiles with an arrow represent the most probable state-transition in chance tiles (as specified by the environment). Very small arrows in each cell encode the policy $\pol$ (the length of each arrow encodes the probability of the corresponding action under the policy, highest probability action is indicated as a red arrow). The heat map is constructed by normalizing the number of visits for each state over $20000$ steps, where actions are sampled from the agent's policy and state-transitions are sampled according to one of three ways: the second row according to the agent's belief over state-transitions $\bel$, in the third and fourth row according to the actual transition probabilities of a friendly and an unfriendly environment respectively. Different columns show different $\alpha$ and $\beta$ cases.
}\label{fig:results1}
\end{figure}

In the first column of Fig. \ref{fig:results1} it can be seen that a stochastic agent ($\alpha = 3.0$) with high model uncertainty and optimistic attitude ($\beta=400$) has a strong preference for the broad corridor in the bottom by assuming favorable transitions for the unknown chance tiles.
This way the agent also avoids the narrow corridors that are unsafe due to the stochasticity of the low-$\alpha$ policy. In the second column of Fig. \ref{fig:results1} with low $\alpha=3$ and high model uncertainty with pessimistic attitude $\beta =-400$, the agent strongly prefers the upper broad corridor because unfavorable transitions are assumed for the chance tiles.
The third column of Fig. \ref{fig:results1} shows a very pessimistic agent ($\beta=-400$) with high precision ($\alpha=11$) that allow the agent to safely choose the shortest distance by selecting the upper narrow corridor without risking any tiles with unknown transitions. The fourth column of Fig. \ref{fig:results1} shows a very optimistic agent ($\beta=400$) with high precision. In this case the agent chooses the shortest distance by selecting the bottom narrow corridor that includes two chance tiles with unknown transition.

\subsection{Updating the Bayesian Posterior $\mu$ with Observations from the Environment}

Similar to model identification adaptive controllers that perform system identification while the system is running \cite{aastrom2013adaptive}, we can use the proposed planning algorithm also in a reinforcement learning setup by updating the Bayesian beliefs about the MDP while executing always the first action and replanning in the next time step.
During the learning phase, the exploration is governed by both factors $\alpha$  and $\beta$, but each factor has a different influence. In particular, lower $\alpha$-values will cause more exploration due to the inherent stochasticity in the agent's action selection, similar to an $\epsilon$-greedy policy. If $\alpha$ is kept fixed through time, this will of course also imply a ``suboptimal'' (i.e. bounded optimal) policy in the long run. In contrast, the parameter $\beta$ governs exploration of states with unknown transition-probabilities more directly and will not have an impact on the agent's performance in the limit, where sufficient data has eliminated model uncertainty. 
We illustrate this with simulations in a grid-world environment where the agent is allowed to update its beliefs $\bayespost$ over the state-transitions every time it enters a chance tile  and receives observation data acquired through interaction with the environment---compare left panels in Figure \ref{fig:results2}. In each step, the agent can then use the updated belief-models for planning the next action.

\begin{figure}
\center
\includegraphics[scale=0.20]{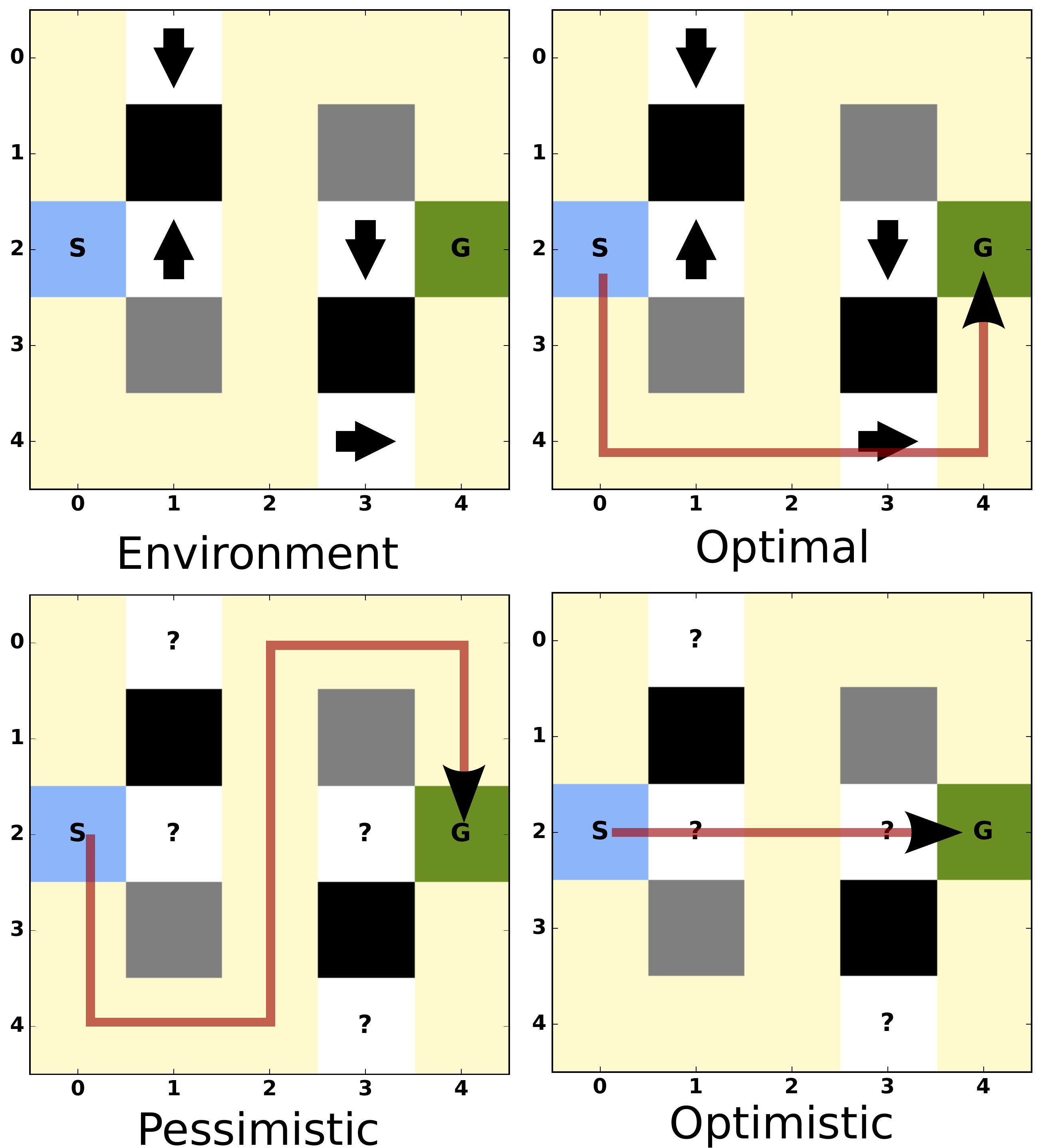}
\includegraphics[scale=0.22]{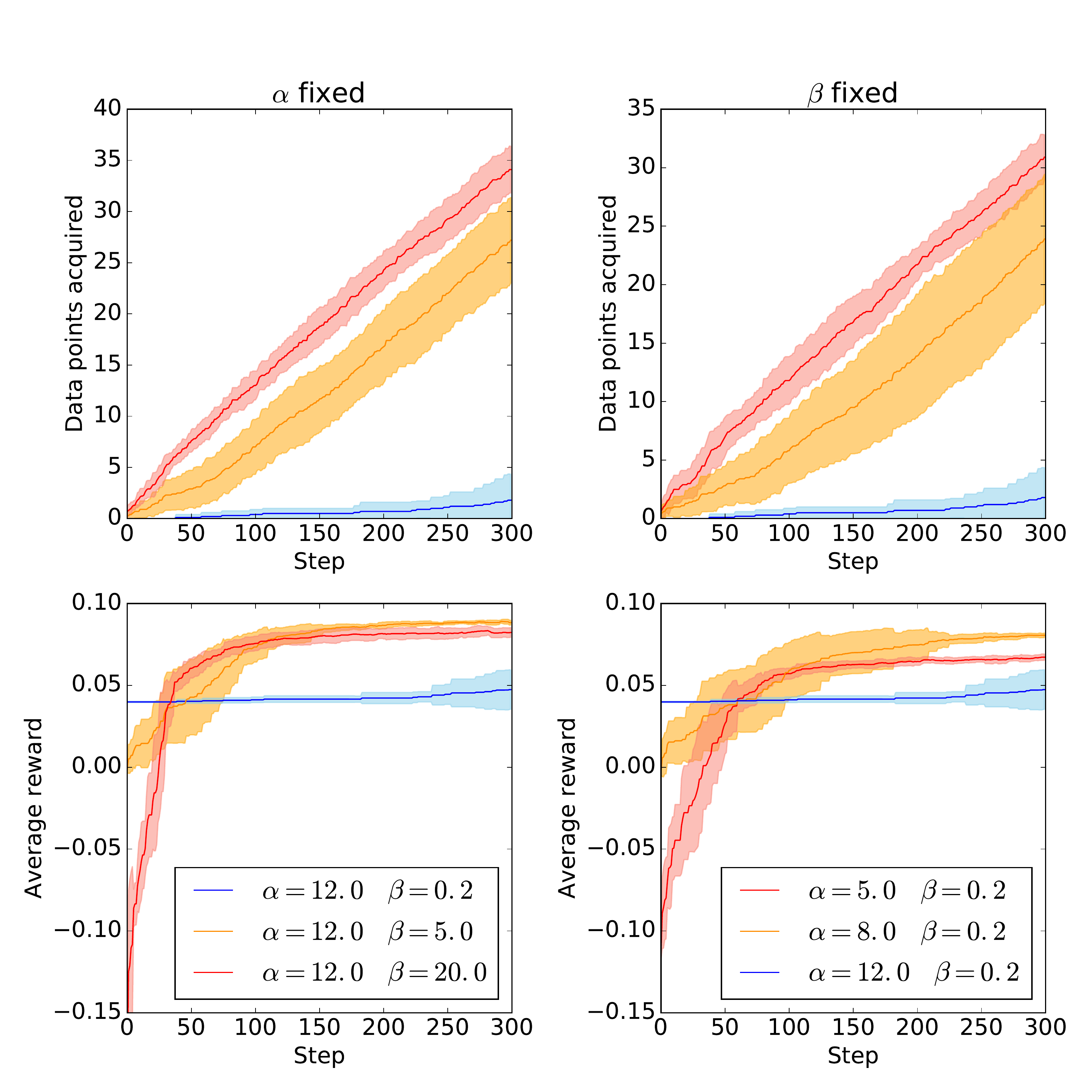}
\caption{The effect of $\alpha$ and $\beta$ when updating beliefs over $300$ interaction steps with the environment. The four panels on the left show the grid-world environment and the pertaining optimal policy if the environment is known. The lower left panels show paths that the agent could take depending on its attitude towards model uncertainty. The panels on the right show the number of acquired data points, that is the number of times a chance tile is entered, and the average reward (bottom panels) for fixed $\alpha $ (varying $\beta$) or fixed $\beta $ (varying $\alpha$). The average reward at each step is computed as follows. Each time the agent observes a state-transition in a chance tile and updates its belief model, $10$ runs of length $2000$ steps  are sampled (using the agent's current belief model).  The average reward (bold lines) and standard-deviation (shaded areas) across these $10$ runs  are shown in the figure.}\label{fig:results2}
\end{figure}

Figure \ref{fig:results2} (right panels) shows the number of data points acquired (each time a chance tile is visited) and the average reward depending on the number of steps that the agent has interacted with the environment.  The panels show several different cases: while keeping $\alpha=12.0$ fixed we test $\beta=(0.2, 5.0, 20.0)$ and while keeping $\beta=0.2$ fixed we test $\alpha=(5.0, 8.0, 12.0)$.
It can be seen that lower $\alpha$ leads to better exploration, but it can also lead to lower performance in the long run---see for example rightmost bottom panel. In contrast, optimistic $\beta$ values can also induce high levels of exploration with the added advantage that in the limit no performance detriment is introduced. However, high $\beta$ values can in general also lead to a detrimental persistence with bad policies, as can be seen for example in the superiority of the low-$\beta$ agent at the very beginning of the learning process.

\section{Discussion and Conclusions}

In this paper we are bringing two strands of research together, namely research on information-theoretic principles of control and decision-making and robustness principles for planning under model uncertainty. We have devised a unified  recursion principle that extends previous generalizations of Bellman's optimality equation and we have shown how to solve this recursion with an iterative scheme that is guaranteed to converge to a unique optimum. In simulations we could demonstrate how such a combination of information-theoretic policy and belief constraints that reflect model uncertainty can be beneficial for agents that act in partially unknown environments.

Most of the research on robust MDPs does not consider information-processing constraints on the policy, but only considers the uncertainty in the transition probabilities by specifying a set of permissible models  such that worst-case scenarios can be computed in order to obtain a robust policy \cite{nilim2005robust,iyengar2005robust}.  Recent extensions of these approaches include more general assumptions regarding the set properties of the permissible models and assumptions regarding the data generation process \cite{wiesemann2013robust}. Our approach falls inside this class of robustness methods that use a restricted set of permissible models, because we extremize the biased belief $\bel$ under the constraint that it has to be within some information bounds measured by the Kullback-Leibler divergence from a reference Bayesian posterior. Contrary to these previous methods, our approach additionally considers robustness arising from the stochasticity in the policy. 

Information-processing constraints on the policy in MDPs have been previously considered in a number of studies \cite{todorov2009efficient,kappen2005linear,peters2010relative,rubin2012trading}, however not in the context of model uncertainty. In these studies a free energy value recursion is derived when restricting the class of policies through the Kullback-Leibler divergence and when disregarding separate information-processing constraints on observations. However, a small number of studies has considered information-processing constraints both for actions and observations. For example, Polani and Tishby \cite{tishby2011information} and  Ortega and Braun \cite{ortega2013thermodynamics}  combine both kinds of information costs. The first cost formalizes an information-processing cost in the policy and the second cost constrains uncertainty arising from the state transitions directly (but crucially not the uncertainty in the latent variables). In both information-processing constraints the cost is determined as a Kullback-Leibler divergence with respect to a reference distribution. 
Specifically, the reference distribution in \cite{tishby2011information} is given by the marginal distributions (which is equivalent to a rate distortion problem) and in \cite{ortega2013thermodynamics} is given by fixed priors. The Kullback-Leibler divergence costs for the observations in these cases essentially correspond to a risk-sensitive objective. 
While there is a relation between risk-sensitive and robust MDPs \cite{shen2014risk,osogami2012robustness,chow2015risk}, the innovation in our approach is at least twofold. First, it allows combining information-processing constraints on the policy with model uncertainty (as formalized by a latent variable). Second, it provides a natural setup to study learning.

The algorithm presented here and Bayesian models in general  \cite{duff2002optimal} are computationally expensive as they have to   compute possibly high-dimensional integrals depending on the number of allowed transitions for action-state pairs. However, there have been tremendous efforts in solving unknown MDPs efficiently, especially by sampling methods \cite{ross2011bayesian,guez2012efficient,guez2013scalable}. An interesting future direction to extend our methodology would therefore be to develop a sampling-based version of Algorithm~\ref{alg:FEiteration} to increase the range of applicability and scalability \cite{ortega2014monte}. Moreover, such sampling methods might allow for reinforcement learning applications, for example by estimating free energies through TD-learning \cite{fox2015g}, or by Thompson sampling approaches \cite{ortega2010minimum,ortega2010bayesian} or other stochastic methods for adaptive control \cite{aastrom2013adaptive}.

\subsubsection{Acknowledgments}
This study was supported by the DFG, Emmy Noether grant BR4164/1-1. The code was developed on top of the RLPy library~\cite{geramifard2015rlpy}.

\bibliographystyle{unsrt}
\bibliography{bibliographyMDP}

\begin{thebibliography}{10}

\bibitem{Bellman:1957}
Richard Bellman.
\newblock {\em Dynamic Programming}.
\newblock Princeton University Press, Princeton, NJ, USA, 1 edition, 1957.

\bibitem{todorov2006linearly}
Emanuel Todorov.
\newblock Linearly-solvable markov decision problems.
\newblock In {\em Advances in neural information processing systems}, pages
  1369--1376, 2006.

\bibitem{todorov2009efficient}
Emanuel Todorov.
\newblock Efficient computation of optimal actions.
\newblock {\em Proceedings of the national academy of sciences},
  106(28):11478--11483, 2009.

\bibitem{braun2011path}
Daniel~A Braun, Pedro~A Ortega, Evangelos Theodorou, and Stefan Schaal.
\newblock Path integral control and bounded rationality.
\newblock In {\em Adaptive Dynamic Programming And Reinforcement Learning
  (ADPRL), 2011 IEEE Symposium on}, pages 202--209. IEEE, 2011.

\bibitem{Broek2010RiskSP}
Bart van~den Broek, Wim Wiegerinck, and Hilbert~J. Kappen.
\newblock Risk sensitive path integral control.
\newblock In {\em UAI}, 2010.

\bibitem{ortega2013thermodynamics}
Pedro~A Ortega and Daniel~A Braun.
\newblock Thermodynamics as a theory of decision-making with
  information-processing costs.
\newblock In {\em Proc. R. Soc. A}, volume 469, page 20120683. The Royal
  Society, 2013.

\bibitem{ortega2014generalized}
Pedro~A Ortega and Daniel~A Braun.
\newblock Generalized thompson sampling for sequential decision-making and
  causal inference.
\newblock {\em Complex Adaptive Systems Modeling}, 2(1):2, 2014.

\bibitem{tishby2011information}
Naftali Tishby and Daniel Polani.
\newblock Information theory of decisions and actions.
\newblock In {\em Perception-action cycle}, pages 601--636. Springer, 2011.

\bibitem{duff2002optimal}
Michael~O'Gordon Duff.
\newblock {\em Optimal Learning: Computational procedures for Bayes-adaptive
  Markov decision processes}.
\newblock PhD thesis, University of Massachusetts Amherst, 2002.

\bibitem{mannor2007bias}
Shie Mannor, Duncan Simester, Peng Sun, and John~N Tsitsiklis.
\newblock Bias and variance approximation in value function estimates.
\newblock {\em Management Science}, 53(2):308--322, 2007.

\bibitem{nilim2005robust}
Arnab Nilim and Laurent El~Ghaoui.
\newblock Robust control of markov decision processes with uncertain transition
  matrices.
\newblock {\em Operations Research}, 53(5):780--798, 2005.

\bibitem{iyengar2005robust}
Garud~N Iyengar.
\newblock Robust dynamic programming.
\newblock {\em Mathematics of Operations Research}, 30(2):257--280, 2005.

\bibitem{wiesemann2013robust}
Wolfram Wiesemann, Daniel Kuhn, and Ber{\c{c}} Rustem.
\newblock Robust markov decision processes.
\newblock {\em Mathematics of Operations Research}, 38(1):153--183, 2013.

\bibitem{szita2008many}
Istv{\'a}n Szita and Andr{\'a}s L{\H{o}}rincz.
\newblock The many faces of optimism: a unifying approach.
\newblock In {\em Proceedings of the 25th international conference on Machine
  learning}, pages 1048--1055. ACM, 2008.

\bibitem{szita2010model}
Istv{\'a}n Szita and Csaba Szepesv{\'a}ri.
\newblock Model-based reinforcement learning with nearly tight exploration
  complexity bounds.
\newblock In {\em Proceedings of the 27th International Conference on Machine
  Learning (ICML-10)}, pages 1031--1038, 2010.

\bibitem{hansen2008robustness}
Lars~Peter Hansen and Thomas~J Sargent.
\newblock {\em Robustness}.
\newblock Princeton university press, 2008.

\bibitem{rubin2012trading}
Jonathan Rubin, Ohad Shamir, and Naftali Tishby.
\newblock Trading value and information in mdps.
\newblock In {\em Decision Making with Imperfect Decision Makers}, pages
  57--74. Springer, 2012.

\bibitem{bertsekas1996neuro}
DP~Bertsekas and JN~Tsitsiklis.
\newblock Neuro-dynamic programming.
\newblock 1996.

\bibitem{strehl2009reinforcement}
Alexander~L Strehl, Lihong Li, and Michael~L Littman.
\newblock Reinforcement learning in finite mdps: Pac analysis.
\newblock {\em The Journal of Machine Learning Research}, 10:2413--2444, 2009.

\bibitem{aastrom2013adaptive}
Karl~J {\AA}str{\"o}m and Bj{\"o}rn Wittenmark.
\newblock {\em Adaptive control}.
\newblock Courier Corporation, 2013.

\bibitem{kappen2005linear}
Hilbert~J Kappen.
\newblock Linear theory for control of nonlinear stochastic systems.
\newblock {\em Physical review letters}, 95(20):200201, 2005.

\bibitem{peters2010relative}
J~Peters, K~M{\"u}lling, Y~Altun, Fox~D Poole, et~al.
\newblock Relative entropy policy search.
\newblock In {\em Twenty-Fourth National Conference on Artificial Intelligence
  (AAAI-10)}, pages 1607--1612. AAAI Press, 2010.

\bibitem{shen2014risk}
Yun Shen, Michael~J Tobia, Tobias Sommer, and Klaus Obermayer.
\newblock Risk-sensitive reinforcement learning.
\newblock {\em Neural computation}, 26(7):1298--1328, 2014.

\bibitem{osogami2012robustness}
Takayuki Osogami.
\newblock Robustness and risk-sensitivity in markov decision processes.
\newblock In {\em Advances in Neural Information Processing Systems}, pages
  233--241, 2012.

\bibitem{chow2015risk}
Yinlam Chow, Aviv Tamar, Shie Mannor, and Marco Pavone.
\newblock Risk-sensitive and robust decision-making: a cvar optimization
  approach.
\newblock In {\em Advances in Neural Information Processing Systems}, pages
  1522--1530, 2015.

\bibitem{ross2011bayesian}
St{\'e}phane Ross, Joelle Pineau, Brahim Chaib-draa, and Pierre Kreitmann.
\newblock A bayesian approach for learning and planning in partially observable
  markov decision processes.
\newblock {\em The Journal of Machine Learning Research}, 12:1729--1770, 2011.

\bibitem{guez2012efficient}
Arthur Guez, David Silver, and Peter Dayan.
\newblock Efficient bayes-adaptive reinforcement learning using sample-based
  search.
\newblock In {\em Advances in Neural Information Processing Systems}, pages
  1025--1033, 2012.

\bibitem{guez2013scalable}
Arthur Guez, David Silver, and Peter Dayan.
\newblock Scalable and efficient bayes-adaptive reinforcement learning based on
  monte-carlo tree search.
\newblock {\em Journal of Artificial Intelligence Research}, pages 841--883,
  2013.

\bibitem{ortega2014monte}
Pedro~A Ortega, Daniel~A Braun, and Naftali Tishby.
\newblock Monte carlo methods for exact \& efficient solution of the
  generalized optimality equations.
\newblock In {\em Robotics and Automation (ICRA), 2014 IEEE International
  Conference on}, pages 4322--4327. IEEE, 2014.

\bibitem{fox2015g}
Roy Fox, Ari Pakman, and Naftali Tishby.
\newblock G-learning: Taming the noise in reinforcement learning via soft
  updates.
\newblock {\em arXiv preprint arXiv:1512.08562}, 2015.

\bibitem{ortega2010minimum}
Pedro~A Ortega and Daniel~A Braun.
\newblock A minimum relative entropy principle for learning and acting.
\newblock {\em Journal of Artificial Intelligence Research}, pages 475--511,
  2010.

\bibitem{ortega2010bayesian}
Pedro~A Ortega and Daniel~A Braun.
\newblock A bayesian rule for adaptive control based on causal interventions.
\newblock In {\em 3d Conference on Artificial General Intelligence (AGI-2010)}.
  Atlantis Press, 2010.

\bibitem{geramifard2015rlpy}
Alborz Geramifard, Christoph Dann, Robert~H Klein, William Dabney, and
  Jonathan~P How.
\newblock Rlpy: A value-function-based reinforcement learning framework for
  education and research.
\newblock {\em Journal of Machine Learning Research}, 16:1573--1578, 2015.

\end{thebibliography}

\end{document}